\documentclass[english]{article}
\usepackage[T1]{fontenc}
\usepackage[utf8]{inputenc}
\usepackage{babel}
\usepackage{float}
\usepackage{amsmath}
\usepackage{amsthm}
\usepackage{amssymb}
\usepackage{graphicx}
\usepackage[authoryear]{natbib}
\usepackage[unicode=true]
 {hyperref}

\makeatletter

\providecommand{\tabularnewline}{\\}
\floatstyle{ruled}
\newfloat{algorithm}{tbp}{loa}
\providecommand{\algorithmname}{Algorithm}
\floatname{algorithm}{\protect\algorithmname}

\theoremstyle{plain}
\newtheorem{thm}{\protect\theoremname}
  \theoremstyle{plain}
  \newtheorem{prop}[thm]{\protect\propositionname}
  \theoremstyle{plain}
  \newtheorem{lem}[thm]{\protect\lemmaname}
  \theoremstyle{plain}
  \newtheorem{cor}[thm]{\protect\corollaryname}

\usepackage[final]{nips_2016}

\usepackage{hyperref}       
\usepackage{booktabs}       
\usepackage{amsfonts}       
\usepackage{nicefrac}       
\usepackage{microtype}      

\@ifundefined{showcaptionsetup}{}{%
 \PassOptionsToPackage{caption=false}{subfig}}
\usepackage{subfig}
\makeatother

  \providecommand{\corollaryname}{Corollary}
  \providecommand{\lemmaname}{Lemma}
  \providecommand{\propositionname}{Proposition}
\providecommand{\theoremname}{Theorem}

\begin{document}

\title{A Simple Practical Accelerated Method for Finite Sums}

\author{Aaron Defazio\\
Ambiata, Sydney Australia}
\maketitle
\begin{abstract}
We describe a novel optimization method for finite sums (such as empirical
risk minimization problems) building on the recently introduced SAGA
method. Our method achieves an accelerated convergence rate on strongly
convex smooth problems. Our method has only one parameter (a step
size), and is radically simpler than other accelerated methods for
finite sums. Additionally it can be applied when the terms are non-smooth,
yielding a method applicable in many areas where operator splitting
methods would traditionally be applied.
\end{abstract}

\section*{Introduction}

A large body of recent developments in optimization have focused on
minimization of convex finite sums of the form:
\[
f(x)=\frac{1}{n}\sum_{i=1}^{n}f_{i}(x),
\]
a very general class of problems including the empirical risk minimization
(ERM) framework as a special case. Any function $h$ can be written
in this form by setting $f_{1}(x)=h(x)$ and $f_{i}=0$ for $i\neq1$,
however when each $f_{i}$ is sufficiently regular in a way that can
be made precise, it is possible to optimize such sums more efficiently
than by treating them as black box functions.

In most cases recently developed methods such as SAG \citep{SAG}
can find an $\epsilon$-minimum faster than either stochastic gradient
descent or accelerated black-box approaches, both in theory and in
practice. We call this class of methods fast incremental gradient
methods (FIG).

FIG methods are randomized methods similar to SGD, however unlike
SGD they are able to achieve linear convergence rates under Lipschitz-smooth
and strong convexity conditions \citep{miso2,finito,svrg,semi}. The
linear rate in the first wave of FIG methods directly depended on
the condition number ($L/\mu$) of the problem, whereas recently several
methods have been developed that depend on the square-root of the
condition number \citep{lan-accel,mairal-catalyst,sdca-accel,svrg-accel-nitanda}.
Analogous to the black-box case, these methods are known as accelerated
methods. 

In this work we develop another accelerated method, which is conceptually
simpler and requires less tuning than existing accelerated methods.
The method we give is a primal approach, however it makes use of a
proximal operator oracle for each $f_{i}$ instead of a gradient oracle,
unlike other primal approaches. The proximal operator is also used
by dual methods such as some variants of SDCA \citep{SDCA}.

\section{Algorithm}

\begin{algorithm}
Pick some starting point $x^{0}$ and step size $\gamma$. Initialize
each $g_{i}^{0}=f_{i}^{\prime}(x^{0}),$ where $f_{i}^{\prime}(x^{0})$
is any gradient/subgradient at $x^{0}$. 

Then at step $k+1$:
\begin{enumerate}
\item Pick index $j$ from $1$ to $n$ uniformly at random.
\item Update $x$:
\[
z_{j}^{k}=x^{k}+\gamma\left[g_{j}^{k}-\frac{1}{n}\sum_{i=1}^{n}g_{i}^{k}\right],
\]
\[
x^{k+1}=\text{prox}_{j}^{\gamma}\left(z_{j}^{k}\right).
\]

\item Update the gradient table: Set $g_{j}^{k+1}=\frac{1}{\gamma}\left(z_{j}^{k}-x^{k+1}\right)$,
and leave the rest of the entries unchanged ($g_{i}^{k+1}=g_{i}^{k}$
for $i\neq j$).
\end{enumerate}
\caption{\label{alg:algorithm1}}
\end{algorithm}

Our algorithm's main step makes use of the proximal operator for a
randomly chosen $f_{i}$. For convenience, we use the following compact
notation:
\[
\text{prox}_{i}^{\gamma}\left(x\right)=\text{arg\ensuremath{\min}}_{y}\left\{ \gamma f_{i}(y)+\frac{1}{2}\left\Vert x-y\right\Vert ^{2}\right\} .
\]
This proximal operator can be computed efficiently or in closed form
in many cases, see Section \ref{sec:implementation} for details.
Like SAGA, we also maintain a table of gradients $g_{i}$, one for
each function $f_{i}$. We denote the state of $g_{i}$ at the end
of step $k$ by $g_{i}^{k}$. The iterate (our guess at the solution)
at the end of step $k$ is denoted $x^{k}.$ The starting iterate
$x^{0}$ may be chosen arbitrarily.

The full algorithm is given as Algorithm \ref{alg:algorithm1}. The
sum of gradients $\frac{1}{n}\sum_{i=1}^{n}g_{i}^{k}$ can be cached
and updated efficiently at each step, and in most cases instead of
storing a full vector for each $g_{i}$, only a single real value
needs to be stored. This is the case for linear regression or binary
classification with logistic loss or hinge loss, in precisely the
same way as for standard SAGA. A discussion of further implementation
details is given in Section \ref{sec:implementation}.

With step size 
\begin{eqnarray*}
\gamma & = & \frac{\sqrt{(n-1)^{2}+4n\frac{L}{\mu}}}{2Ln}-\frac{1-\frac{1}{n}}{2L},
\end{eqnarray*}
the expected convergence rate in terms of squared distance to the
solution is given by: 
\[
E\left\Vert x^{k}-x^{*}\right\Vert ^{2}\leq\left(1-\frac{\mu\gamma}{1+\mu\gamma}\right)^{k}\frac{\mu+L}{\mu}\left\Vert x^{0}-x^{*}\right\Vert ^{2},
\]
when each $f_{i}:\mathbb{R}^{d}\rightarrow\mathbb{R}$ is $L$-smooth
and $\mu$-strongly convex. See \citet{nes-book} for definitions
of these conditions. Using big-O notation, the number of steps required
to reduce the distance to the solution by a factor $\epsilon$ is:
\[
k=O\left(\left(\sqrt{\frac{nL}{\mu}}+n\right)\log\left(\frac{1}{\epsilon}\right)\right),
\]
as $\epsilon\rightarrow0$. This rate matches the lower bound known
for this problem \citep{lan-accel} under the gradient oracle. We
conjecture that this rate is optimal under the proximal operator oracle
as well. Unlike other accelerated approaches though, we have only
a single tunable parameter (the step size $\gamma$), and the algorithm
doesn't need knowledge of $L$ or $\mu$ except for their appearance
in the step size.

Compared to the $O\left(\left(L/\mu+n\right)\log\left(1/\epsilon\right)\right)$
rate for SAGA and other non-accelerated FIG methods, accelerated FIG
methods are significantly faster when $n$ is small compared to $L/\mu$,
however for $n\geq L/\mu$ the performance is essentially the same.
All known FIG methods hit a kind of wall at $n\approx L/\mu$, where
they decrease the error at each step by no more than $1-\frac{1}{n}$.
Indeed, when $n\geq L/\mu$ the problem is so well conditioned so
as to be easy for any FIG method to solve it efficiently. This is
sometimes called the big data setting \citep{finito}.

Our convergence rate can also be compared to that of optimal first-order
black box methods, which have rates of the form $k=O\left(\left(\sqrt{L/\mu}\right)\log\left(1/\epsilon\right)\right)$
per epoch equivalent. We are able to achieve a $\sqrt{n}$ speedup
on a per-epoch basis, for $n$ not too large. Of course, all of the
mentioned rates are significantly better than the $O\left(\left(L/\mu\right)\log\left(1/\epsilon\right)\right)$
rate of gradient descent. 

For non-smooth but strongly convex problems, we prove a $1/\epsilon$-type
rate under a standard iterate averaging scheme. This rate does not
require the use of decreasing step sizes, so our algorithm requires
less tuning than other primal approaches on non-smooth problems.

\section{Relation to other approaches}

Our method is most closely related to the SAGA method. To make the
relation clear, we may write our method's main step as:
\[
x^{k+1}=x^{k}-\gamma\left[f_{j}^{\prime}(x^{k+1})-g_{j}^{k}+\frac{1}{n}\sum_{i=1}^{n}g_{i}^{k}\right],
\]
whereas SAGA has a step of the form:
\[
x^{k+1}=x^{k}-\gamma\left[f_{j}^{\prime}(x^{k})-g_{j}^{k}+\frac{1}{n}\sum_{i=1}^{n}g_{i}^{k}\right].
\]
The difference is the point at which the gradient of $f_{j}$ is evaluated
at. The proximal operator has the effect of evaluating the gradient
at $x^{k+1}$ instead of $x^{k}$. While a small difference on the
surface, this change has profound effects. It allows the method to
be applied directly to non-smooth problems using fixed step sizes,
a property not shared by SAGA or other primal FIG methods. Additionally,
it allows for much larger step sizes to be used, which is why the
method is able to achieve an accelerated rate.

It is also illustrative to look at how the methods behave at $n=1$.
SAGA degenerates into regular gradient descent, whereas our method
becomes the proximal-point method \citep{rockafellar1976monotone}:
\[
x^{k+1}=\text{prox}_{\gamma f}(x^{k}).
\]
The proximal point method has quite remarkable properties. For strongly
convex problems, it converges \emph{for any} $\gamma>0$ at a linear
rate. The downside being the inherent difficulty of evaluating the
proximal operator. For the $n=2$ case, if each term is an indicator
function for a convex set, our algorithm matches Dykstra's projection
algorithm if we take $\gamma=2$ and use cyclic instead of random
steps.

\subsection*{Accelerated incremental gradient methods}

Several acceleration schemes have been recently developed as extensions
of non-accelerated FIG methods. The earliest approach developed was
the ASDCA algorithm \citep{accel-sdca,sdca-accel}. The general approach
of applying the proximal-point method as the outer-loop of a double-loop
scheme has been dubbed the Catalyst algorithm \citet{mairal-catalyst}.
It can be applied to accelerate any FIG method. Recently a very interesting
primal-dual approach has been proposed by \citet{lan-accel}. All
of the prior accelerated methods are significantly more complex than
the approach we propose, and have more complex proofs.

\section{Theory}

\subsection{Proximal operator bounds}

In this section we rehash some simple bounds from proximal operator
theory that we will use in this work. Define the short-hand $p_{\gamma f}(x)=\text{prox}_{\gamma f}(x)$,
and let $g_{\gamma f}(x)=\frac{1}{\gamma}\left(x-p_{\gamma f}(x)\right)$,
so that $p_{\gamma f}(x)=x-\gamma g_{\gamma f}(x)$. Note that $g_{\gamma f}(x)$
is a subgradient of $f$ at the point $p_{\gamma f}(x)$. This relation
is known as the\textbf{ optimality condition} of the proximal operator.
Note that proofs for the following two propositions are in the supplementary
material.
\begin{prop}
\label{thm:firm-nonexpansiveness}\textbf{(Strengthening firm non-expansiveness
under strong convexity)} For any $x,y\in\mathbb{R}^{d}$, and any
convex function $f:\mathbb{R}^{d}\rightarrow\mathbb{R}$ with strong
convexity constant $\mu\geq0$, 
\[
\left\langle x-y,p_{\gamma f}(x)-p_{\gamma f}(y)\right\rangle \geq(1+\mu\gamma)\left\Vert p_{\gamma f}(x)-p_{\gamma f}(y)\right\Vert ^{2}.
\]

In operator theory this property is known as $(1+\mu\gamma)$-cocoerciveness
of $p_{\gamma f}$.
\end{prop}

\begin{prop}
\textbf{(Moreau decomposition)} For any $x\in\mathbb{R}^{d}$, and
any convex function $f:\mathbb{R}^{d}\rightarrow\mathbb{R}$ with
Fenchel conjugate $f^{*}$ :
\begin{equation}
p_{\gamma f}(x)=x-\gamma p_{\frac{1}{\gamma}f^{*}}(x/\gamma).\label{eq:Moreau}
\end{equation}

Recall our definition of $g_{\gamma f}(x)=\frac{1}{\gamma}\left(x-p_{\gamma f}(x)\right)$
also. After combining, the following relation thus holds between the
proximal operator of the conjugate $f^{*}$ and $g_{\gamma f}$:
\begin{equation}
p_{\frac{1}{\gamma}f^{*}}(x/\gamma)=\frac{1}{\gamma}\left(x-p_{\gamma f}(x)\right)=g_{\gamma f}(x).\label{eq:gconj}
\end{equation}
\end{prop}
\begin{thm}
\label{thm:g-bound}For any $x,y\in\mathbb{R}^{d}$, and any convex
$L$-smooth function $f:\mathbb{R}^{d}\rightarrow\mathbb{R}$:

\[
\left\langle g_{\gamma f}(x)-g_{\gamma f}(y),x-y\right\rangle \geq\gamma\left(1+\frac{1}{L\gamma}\right)\left\Vert g_{\gamma f}(x)-g_{\gamma f}(y)\right\Vert ^{2},
\]
\end{thm}
\begin{proof}
We will apply cocoerciveness of the proximal operator of $f^{*}$
as it appears in the decomposition. Note that L-smoothness of $f$
implies $1/L$-strong convexity of $f^{*}$. In particular we apply
it to the points $\frac{1}{\gamma}x$ and $\frac{1}{\gamma}y$:
\[
\left\langle p_{\frac{1}{\gamma}f^{*}}(\frac{1}{\gamma}x)-p_{\frac{1}{\gamma}f^{*}}(\frac{1}{\gamma}y),\frac{1}{\gamma}x-\frac{1}{\gamma}y\right\rangle \geq\left(1+\frac{1}{L\gamma}\right)\left\Vert p_{\frac{1}{\gamma}f^{*}}(\frac{1}{\gamma}x)-p_{\frac{1}{\gamma}f^{*}}(\frac{1}{\gamma}y)\right\Vert ^{2}.
\]

Pulling $\frac{1}{\gamma}$ from the right side of the inner product
out, and plugging in Equation \ref{eq:gconj}, gives the result.
\end{proof}

\subsection{Notation}

Let $x^{*}$ be the unique minimizer (due to strong convexity) of
$f$. In addition to the notation used in the description of the algorithm,
we also fix a set of subgradients $g_{j}^{*}$, one for each of $f_{j}$
at $x^{*}$, chosen such that $\sum_{j=1}^{n}g_{j}^{*}=0$. We also
define $v_{j}=x^{*}+\gamma g_{j}^{*}.$ Note that at the solution
$x^{*}$, we want to apply a proximal step for component $j$ of the
form:
\[
x^{*}=\text{prox}_{j}^{\gamma}\left(x^{*}+\gamma g_{j}^{*}\right)=\text{prox}_{j}^{\gamma}\left(v_{j}\right).
\]
\begin{table*}
\noindent \centering{}%
\begin{tabular}{|c|c|c|}
\hline 
Notation & Description & Additional relation\tabularnewline
\hline 
\hline 
$x^{k}$ & Current iterate at step $k$ & $x^{k}\in R^{d}$\tabularnewline
\hline 
$x^{*}$ & Solution & $x^{*}\in R^{d}$\tabularnewline
\hline 
$\gamma$  & Step size & \tabularnewline
\hline 
$p_{\gamma f}(x)$ & Short-hand in results for generic $f$ & $p_{\gamma f}(x)=\text{prox}_{\gamma f}(x)$\tabularnewline
\hline 
$\text{prox}_{i}^{\gamma}\left(x\right)$ & Proximal operator of $\gamma f_{i}$ at $x$ & $=\text{arg\ensuremath{\min}}_{y}\left\{ \gamma f_{i}(y)+\frac{1}{2}\left\Vert x-y\right\Vert ^{2}\right\} $\tabularnewline
\hline 
$g_{i}^{k}$ & A stored subgradient of $f_{i}$ as seen at step $k$ & \tabularnewline
\hline 
$g_{i}^{*}$ & A subgradient of $f_{i}$ at $x^{*}$ & $\sum_{i=1}^{n}g_{i}^{*}=0$\tabularnewline
\hline 
$v_{i}$ & $v_{i}=x^{*}+\gamma g_{i}^{*}$ & $x^{*}=\text{prox}_{i}^{\gamma}\left(v_{i}\right)$\tabularnewline
\hline 
$j$ & Chosen component index (random variable) & \tabularnewline
\hline 
$z_{j}^{k}$ & $z_{j}^{k}=x^{k}+\gamma\left[g_{j}^{k}-\frac{1}{n}\sum_{i=1}^{n}g_{i}^{k}\right]$ & $x_{j}^{k+1}=\text{prox}_{j}^{\gamma}\left(z_{j}^{k}\right)$\tabularnewline
\hline 
\end{tabular}\caption{Notation quick reference}
\end{table*}

\begin{lem}
(Technical lemma needed by main proof) \label{lem:inner-lemma}Under
Algorithm \ref{alg:algorithm1}, taking the expectation over the random
choice of $j$, conditioning on $x^{k}$ and each $g_{i}^{k}$, allows
us to bound the following inner product at step $k$:
\begin{alignat*}{1}
 & E\left\langle \gamma\left[g_{j}^{k}-\frac{1}{n}\sum_{i=1}^{n}g_{i}^{k}\right]-\gamma g_{j}^{*},\left(x^{k}-x^{*}\right)+\gamma\left[g_{j}^{k}-\frac{1}{n}\sum_{i=1}^{n}g_{i}^{k}\right]-\gamma g_{j}^{*}\right\rangle \\
 & \leq\gamma^{2}\frac{1}{n}\sum_{i=1}^{n}\left\Vert g_{i}^{k}-g_{i}^{*}\right\Vert ^{2}.
\end{alignat*}
The proof is in the supplementary material.
\end{lem}

\subsection{Main result }
\begin{thm}
\textbf{(single step Lyapunov descent)} \label{thm:main-thm}We define
the Lyapunov function $T^{k}$ of our algorithm (Point-SAGA) at step
$k$ as:
\[
T^{k}=\frac{c}{n}\sum_{i=1}^{n}\left\Vert g_{i}^{k}-g_{i}^{*}\right\Vert ^{2}+\left\Vert x^{k}-x^{*}\right\Vert ^{2},
\]

for $c=1/\mu L$. Then using step size $\gamma=\frac{\sqrt{(n-1)^{2}+4n\frac{L}{\mu}}}{2Ln}-\frac{1-\frac{1}{n}}{2L}$,
the expectation of $T^{k+1}$, over the random choice of $j$, conditioning
on $x^{k}$ and each $g_{i}^{k}$, is:
\[
E\left[T^{k+1}\right]\leq\left(1-\kappa\right)T^{k}\quad\text{for }\kappa=\frac{\mu\gamma}{1+\mu\gamma},
\]
when each $f_{i}:\mathbb{R}^{d}\rightarrow\mathbb{R}$ is $L$-smooth
and $\mu$-strongly convex and $0<\mu<L$. This is the same Lyapunov
function as used by \citet{lacoste-neighbors}.\end{thm}
\begin{proof}
Term 1 of $T^{k+1}$ is straight-forward to simplify:
\begin{gather*}
\frac{c}{n}E\sum_{i=1}^{n}\left\Vert g_{i}^{k+1}-g_{i}^{*}\right\Vert ^{2}=\left(1-\frac{1}{n}\right)\frac{c}{n}\sum_{i=1}^{n}\left\Vert g_{i}^{k}-g_{i}^{*}\right\Vert ^{2}+\frac{c}{n}E\left\Vert g_{j}^{k+1}-g_{j}^{*}\right\Vert ^{2}.
\end{gather*}
For term $2$ of $T^{k+1}$ we start by applying cocoerciveness (Theorem
\ref{thm:firm-nonexpansiveness}):
\[
(1+\mu\gamma)E\left\Vert x^{k+1}-x^{*}\right\Vert ^{2}
\]
\begin{eqnarray*}
 & = & (1+\mu\gamma)E\left\Vert \text{prox}_{j}^{\gamma}(z_{j}^{k})-\text{prox}_{j}^{\gamma}(v_{j})\right\Vert ^{2}\\
 & \leq & E\left\langle \text{prox}_{j}^{\gamma}(z_{j}^{k})-\text{prox}_{j}^{\gamma}(v_{j}),z_{j}^{k}-v_{j}\right\rangle \\
 & = & E\left\langle x^{k+1}-x^{*}\,,\,z_{j}^{k}-v_{j}\right\rangle .
\end{eqnarray*}
Now we add and subtract $x^{k}:$
\begin{eqnarray*}
 & = & E\left\langle x^{k+1}-x^{k}+x^{k}-x^{*}\,,\,z_{j}^{k}-v_{j}\right\rangle \\
 & = & E\left\langle x^{k}-x^{*}\,,\,z_{j}^{k}-v_{j}\right\rangle +E\left\langle x^{k+1}-x^{k}\,,\,z_{j}^{k}-v_{j}\right\rangle \\
 & = & \left\Vert x^{k}-x^{*}\right\Vert ^{2}+E\left\langle x^{k+1}-x^{k}\,,\,z_{j}^{k}-v_{j}\right\rangle ,
\end{eqnarray*}
where we have pulled out the quadratic term by using $E[z_{j}^{k}-v_{j}]=x^{k}-x^{*}$
(we can take the expectation since the left hand side of the inner
product doesn't depend on $j$). We now expand $E\left\langle x^{k+1}-x^{k}\,,\,z_{j}^{k}-v_{j}\right\rangle $
further:
\[
E\left\langle x^{k+1}-x^{k}\,,\,z_{j}^{k}-v_{j}\right\rangle 
\]
\begin{alignat}{1}
 & =E\left\langle x^{k+1}-\gamma g_{j}^{*}+\gamma g_{j}^{*}-x^{k}\,,\,z_{j}^{k}-v_{j}\right\rangle \nonumber \\
 & =E\left\langle x^{k}-\gamma g_{j}^{k+1}+\gamma\left[g_{j}^{k}-\frac{1}{n}\sum_{i=1}^{n}g_{i}^{k}\right]-\gamma g_{j}^{*}+\gamma g_{j}^{*}-x^{k},\right.\nonumber \\
 & \quad\quad\quad\quad\quad\left.\left(x^{k}-x^{*}\right)+\gamma\left[g_{j}^{k}-\frac{1}{n}\sum_{i=1}^{n}g_{i}^{k}\right]-\gamma g_{j}^{*}\right\rangle .
\end{alignat}
We further split the left side of the inner product to give two separate
inner products:
\begin{alignat}{1}
 & =E\left\langle \gamma\left[g_{j}^{k}-\frac{1}{n}\sum_{i=1}^{n}g_{i}^{k}\right]-\gamma g_{j}^{*},\left(x^{k}-x^{*}\right)+\gamma\left[g_{j}^{k}-\frac{1}{n}\sum_{i=1}^{n}g_{i}^{k}\right]-\gamma g_{j}^{*}\right\rangle \nonumber \\
 & +E\left\langle \gamma g_{j}^{*}-\gamma g_{j}^{k+1},\left(x^{k}-x^{*}\right)+\gamma\left[g_{j}^{k}-\frac{1}{n}\sum_{i=1}^{n}g_{i}^{k}\right]-\gamma g_{j}^{*}\right\rangle .\label{eq:m2}
\end{alignat}
The first inner product in Equation \ref{eq:m2} is the quantity we
bounded in Lemma \ref{lem:inner-lemma} by $\gamma^{2}\frac{1}{n}\sum_{i=1}^{n}\left\Vert g_{i}^{k}-g_{i}^{*}\right\Vert ^{2}$.
The second inner product in Equation \ref{eq:m2}, can be simplified
using Theorem \ref{thm:g-bound} (note the right side of the inner
product is equal to $z_{j}^{k}-v_{j}$):
\[
-\gamma E\left\langle g_{j}^{k+1}-g_{j}^{*}\,,\,z_{j}^{k}-v_{j}\right\rangle \leq-\gamma^{2}\left(1+\frac{1}{L\gamma}\right)E\left\Vert g_{j}^{k+1}-g_{j}^{*}\right\Vert ^{2}.
\]
Combing these gives the following bound on $(1+\mu\gamma)E\left\Vert x^{k+1}-x^{*}\right\Vert ^{2}$:
\[
(1+\mu\gamma)E\left\Vert x^{k+1}-x^{*}\right\Vert ^{2}\leq\left\Vert x^{k}-x^{*}\right\Vert ^{2}+\gamma^{2}\frac{1}{n}\sum_{i=1}^{n}\left\Vert g_{i}^{k}-g_{i}^{*}\right\Vert ^{2}-\gamma^{2}\left(1+\frac{1}{L\gamma}\right)E\left\Vert g_{j}^{k+1}-g_{j}^{*}\right\Vert ^{2}.
\]
Define $\alpha=\frac{1}{1+\mu\gamma}=1-\kappa$, where $\kappa=\frac{\mu\gamma}{1+\mu\gamma}$.
Now we multiply the above inequality through by $\alpha$ and combine
with the rest of the Lyapunov function, giving:
\begin{alignat*}{1}
E\left[T^{k+1}\right] & \leq T^{k}+\left(\alpha\gamma^{2}-\frac{c}{n}\right)\frac{1}{n}\sum_{i}^{n}\left\Vert g_{i}^{k}-g_{i}^{*}\right\Vert ^{2}\\
 & +\left(\frac{c}{n}-\alpha\gamma^{2}-\frac{\alpha\gamma}{L}\right)E\left\Vert g_{j}^{k+1}-g_{j}^{*}\right\Vert ^{2}-\kappa E\left\Vert x^{k}-x^{*}\right\Vert ^{2}.
\end{alignat*}
We want an $\alpha$ convergence rate, so we pull out the required
terms:
\begin{alignat*}{1}
E\left[T^{k+1}\right] & \leq\alpha T^{k}+\left(\alpha\gamma^{2}+\kappa c-\frac{c}{n}\right)\frac{1}{n}\sum_{i}^{n}\left\Vert g_{i}^{k}-g_{i}^{*}\right\Vert ^{2}\\
 & +\left(\frac{c}{n}-\alpha\gamma^{2}-\frac{\alpha\gamma}{L}\right)E\left\Vert g_{j}^{k+1}-g_{j}^{*}\right\Vert ^{2}.
\end{alignat*}
Now to complete the proof we note that $c=1/\mu L$ and $\gamma=\frac{\sqrt{(n-1)^{2}+4n\frac{L}{\mu}}}{2Ln}-\frac{1-\frac{1}{n}}{2L}$
ensure that both terms inside the round brackets are non-positive,
giving $ET^{k+1}\leq\alpha T^{k}$. These constants were found by
equating the equations in the brackets to zero, and solving with respect
to the two unknowns, $\gamma$ and $c$. It is easy to verify that
$\gamma$ is always positive, as a consequence of the condition number
$L/\mu$ always being at least 1.\end{proof}
\begin{cor}
\textbf{(Smooth case)} Chaining Theorem \ref{thm:main-thm} gives
a convergence rate for Point-SAGA at step $k$ under the constants
given in Theorem \ref{thm:main-thm} of:
\[
E\left\Vert x^{k}-x^{*}\right\Vert ^{2}\leq\left(1-\kappa\right)^{k}\frac{\mu+L}{\mu}\left\Vert x^{0}-x^{*}\right\Vert ^{2},
\]
if each $f_{i}:\mathbb{R}^{d}\rightarrow\mathbb{R}$ is $L$-smooth
and $\mu$-strongly convex.\end{cor}
\begin{thm}
\textbf{(Non-smooth case)} Suppose each $f_{i}:\mathbb{R}^{d}\rightarrow\mathbb{R}$
is $\mu$-strongly convex, $\left\Vert g_{i}^{0}-g_{i}^{*}\right\Vert \leq B$
and $\left\Vert x^{0}-x^{*}\right\Vert \leq R$. Then after $k$ iterations
of Point-SAGA with step size $\gamma=R/B\sqrt{n}$: 
\[
E\left\Vert \bar{x}^{k}-x^{*}\right\Vert ^{2}\leq2\frac{\sqrt{n}\left(1+\mu\left(R/B\sqrt{n}\right)\right)}{\mu k}RB,
\]
where $\bar{x}^{k}=\frac{1}{k}E\sum_{t=1}^{k}x^{t}.$ The proof of
this theorem is included in the supplementary material.
\end{thm}

\section{Implementation }

\label{sec:implementation}Care must be taken for efficient implementation,
particularly in the sparse gradient case. We discuss the key points
below. A fast Cython implementation is available on the author's website
incorporating these techniques.\vspace{-3mm}

\begin{description}
\item [{Proximal\ operators}] For the most common binary classification
and regression methods, implementing the proximal operator is straight-forward.
We include details of the computation of the proximal operators for
the hinge, square and logistic losses in the supplementary material.
The logistic loss does not have a closed form proximal operator, however
it may be computed very efficiently in practice using Newton's method
on a 1D subproblem. For problems of a non-trivial dimensionality the
cost of the dot products in the main step is much greater than the
cost of the proximal operator evaluation. We also detail how to handle
a quadratic regularizer within each term's prox operator, which has
a closed form in terms of the unregularized prox operator.\vspace{-1mm}

\item [{Initialization}] Instead of setting $g_{i}^{0}=f_{i}^{\prime}(x^{0})$
before commencing the algorithm, we recommend using $g_{i}^{0}=0$
instead. This avoids the cost of a initial pass over the data. In
practical effect this is similar to the SDCA initialization of each
dual variable to 0.
\end{description}

\section{Experiments\vspace{-1mm}
}

We tested our algorithm which we call Point-SAGA against SAGA \citep{adefazio-nips2014},
SDCA \citep{SDCA}, Pegasos/SGD \citep{pegasos} and the catalyst
acceleration scheme \citep{mairal-catalyst}. SDCA was chosen as the
inner algorithm for the catalyst scheme as it doesn't require a step-size,
making it the most practical of the variants. Catalyst applied to
SDCA is essentially the same algorithm as proposed in \citet{sdca-accel}.
A single inner epoch was used for each SDCA invocation. Accelerated
MISO as well as the primal-dual FIG method \citep{lan-accel} were
excluded as we wanted to test on sparse problems and they are not
designed to take advantage of sparsity. The step-size parameter for
each method ($\kappa$ for catalyst-SDCA) was chosen using a grid
search of powers of $2$. The step size that gives the lowest error
at the final epoch is used for each method.

We selected a set of commonly used datasets from the LIBSVM repository
\citep{libsvm}. The pre-scaled versions were used when available.
Logistic regression with $L_{2}$ regularization was applied to each
problem. The $L_{2}$ regularization constant for each problem was
set by hand to ensure $f$ was not in the big data regime $n\geq L/\mu$;
as noted above, all the methods perform essentially the same when
$n\geq L/\mu$. The constant used is noted beneath each plot. Open
source code to exactly replicate the experimental results is available
at \href{https://github.com/adefazio/point-saga}{https://github.com/adefazio/point-saga}.
\vspace{-2mm}

\paragraph*{Algorithm scaling with respect to $n$}

The key property that distinguishes accelerated FIG methods from their
non-accelerated counterparts is their performance scaling with respect
to the dataset size. For large datasets on well-conditioned problems
we expect from the theory to see little difference between the methods.
To this end, we ran experiments including versions of the datasets
subsampled randomly without replacement in 10\% and 5\% increments,
in order to show the scaling with $n$ empirically. The same amount
of regularization was used for each subset.

Figure \ref{fig:results} shows the function value sub-optimality
for each dataset-subset combination. We see that in general accelerated
methods dominate the performance of their non-accelerated counter-parts.
Both SDCA and SAGA are much slower on some datasets comparatively
than others. For example, SDCA is very slow on the 5 and 10\% COVTYPE
datasets, whereas both SAGA and SDCA are much slower than the accelerated
methods on the AUSTRALIAN dataset. These differences reflect known
properties of the two methods. SAGA is able to adapt to inherent strong
convexity while SDCA can be faster on very well-conditioned problems.

There is no clear winner between the two accelerated methods, each
gives excellent results on each problem. The Pegasos (stochastic gradient
descent) algorithm with its slower than linear rate is a clear loser
on each problem, almost appearing as an almost horizontal line on
the log scale of these plots.

\begin{figure*}
\subfloat[COVTYPE $\mu=2\times10^{-6}$ : 5\%, 10\%, 100\% subsets]{\includegraphics[width=4.8cm]{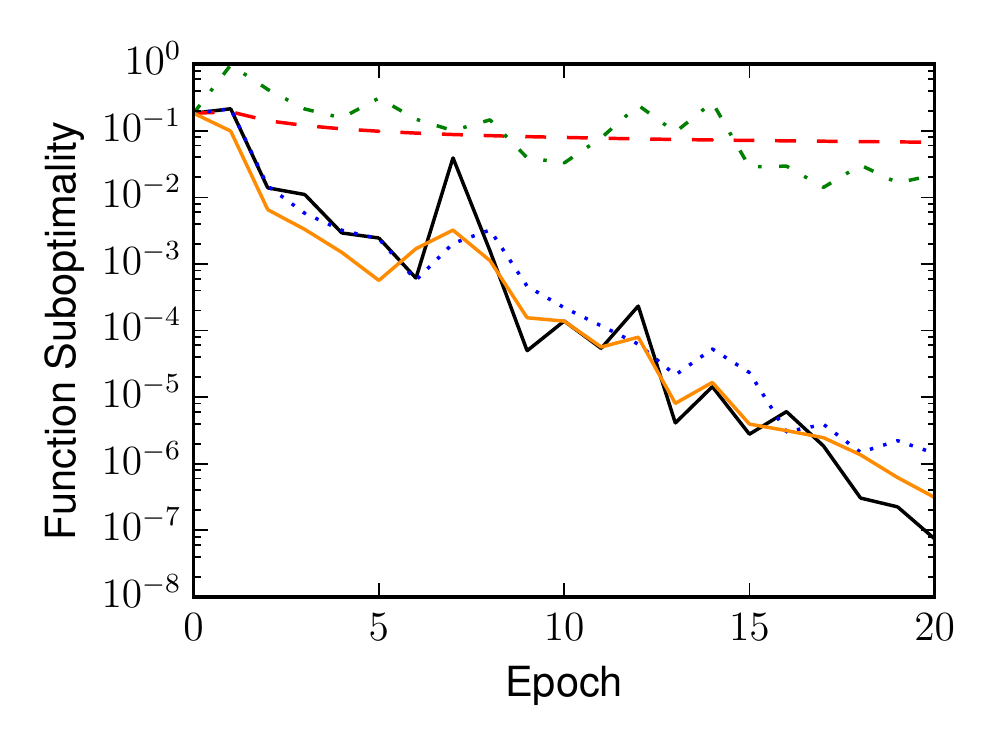}\includegraphics[width=4.8cm]{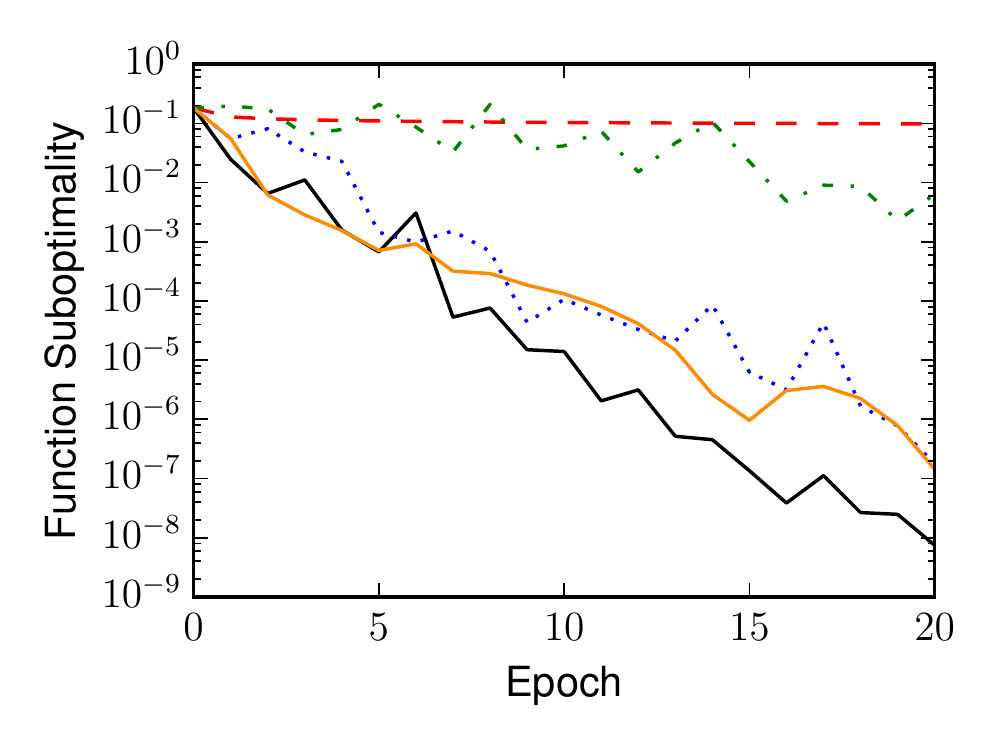}\includegraphics[width=4.8cm]{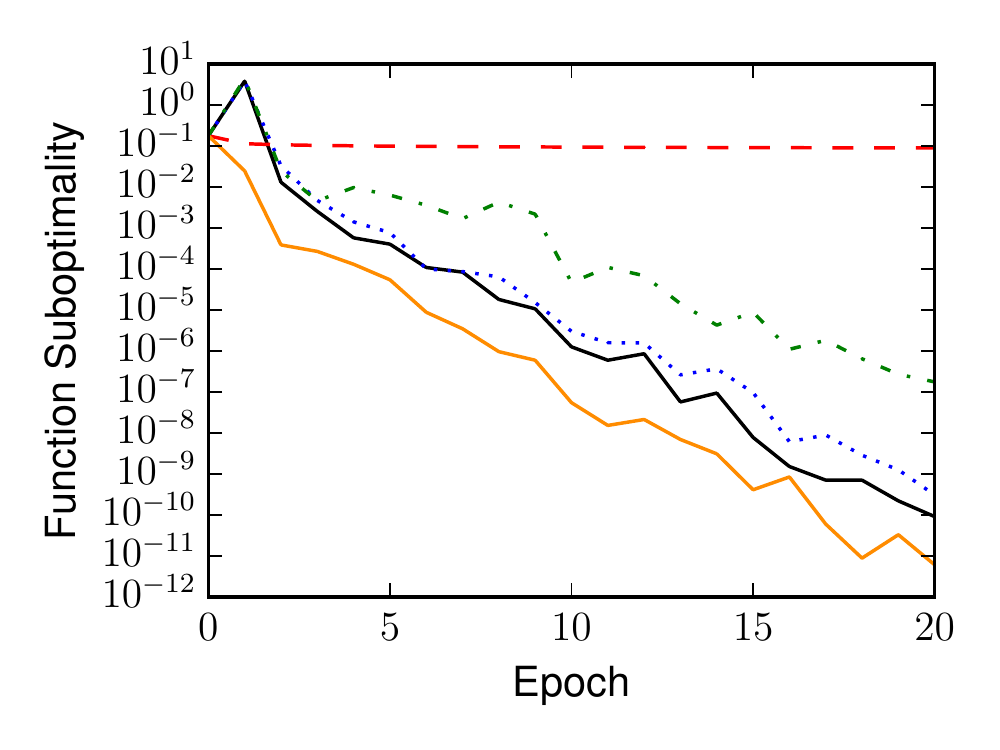}

}

\subfloat[AUSTRALIAN $\mu=10^{-4}$: 5\%, 10\%, 100\% subsets]{\includegraphics[width=4.8cm]{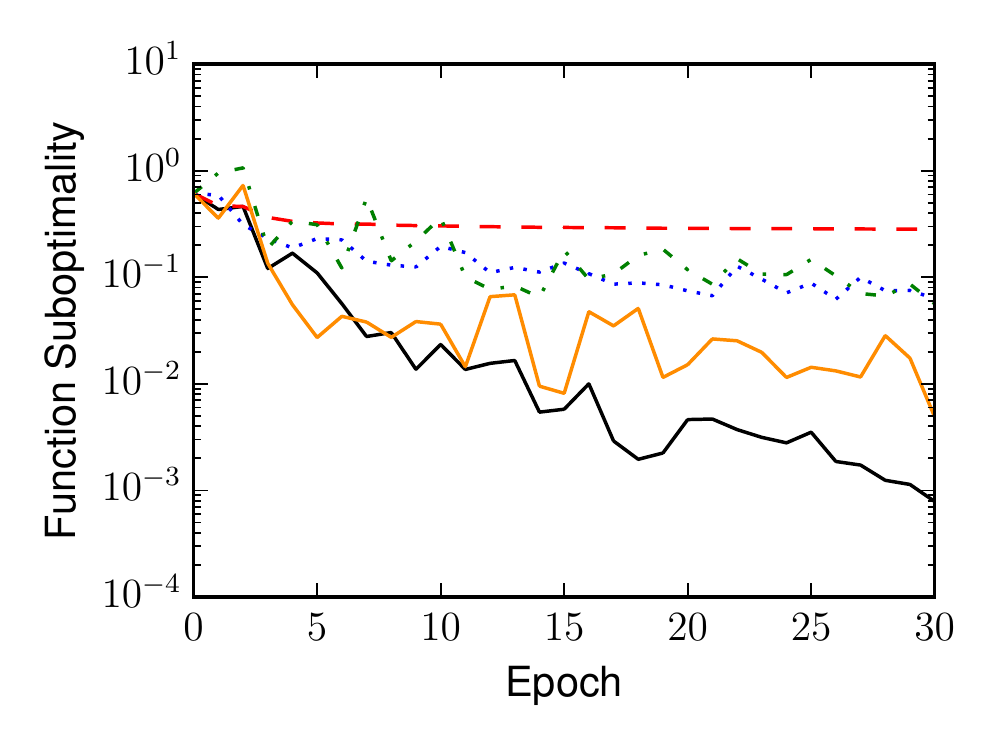}\includegraphics[width=4.8cm]{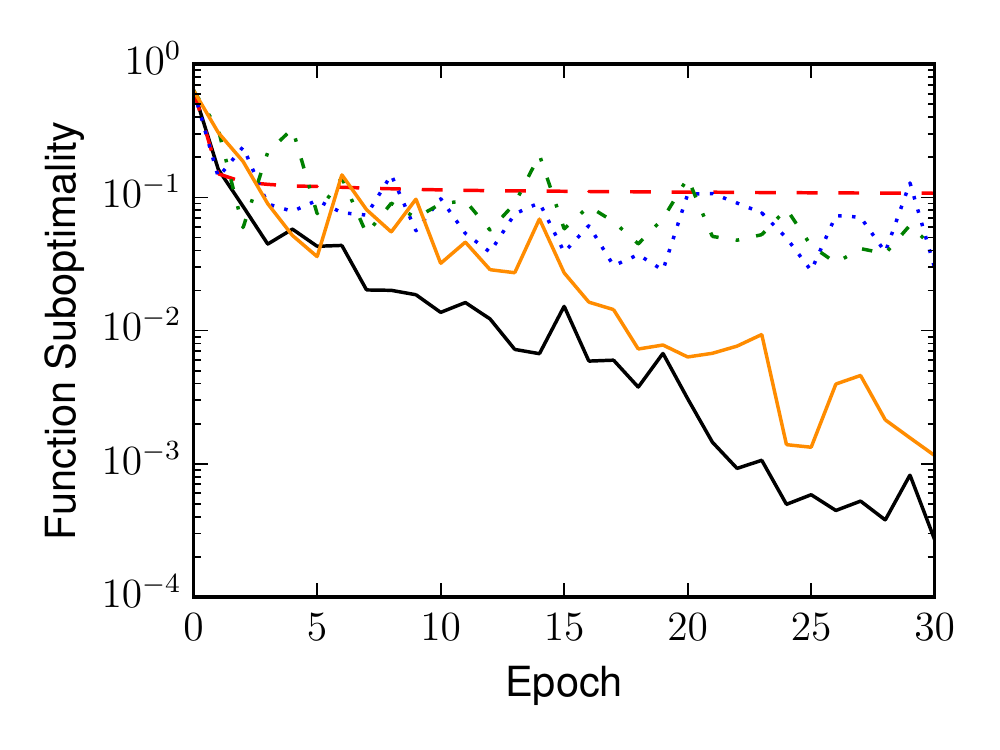}\includegraphics[width=4.8cm]{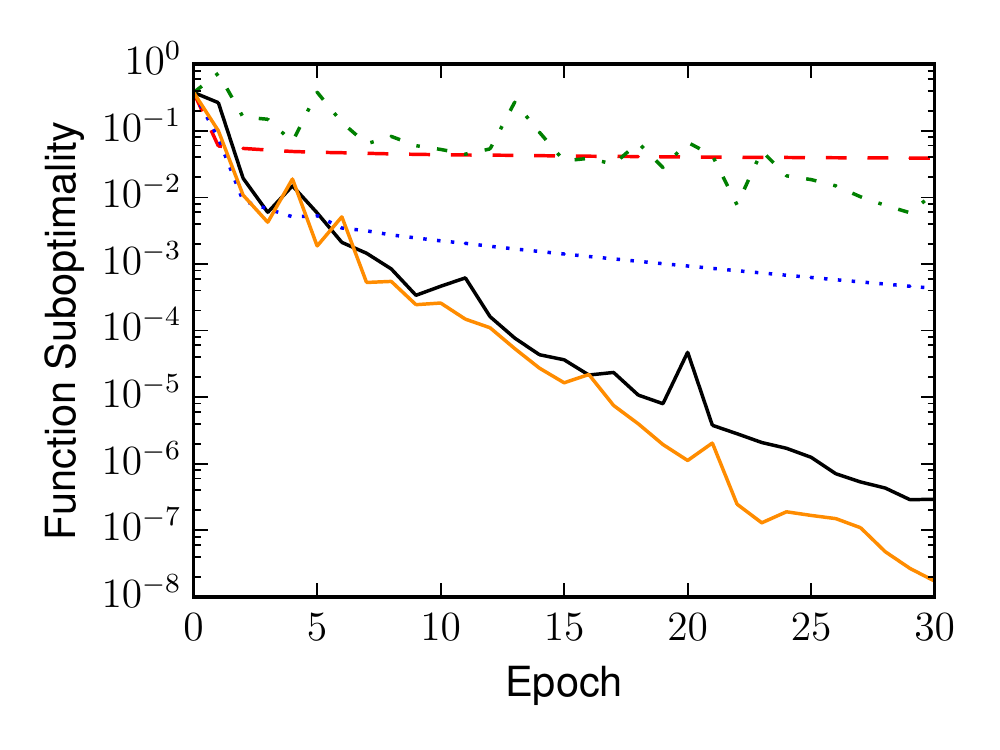}

}

\subfloat[MUSHROOMS $\mu=10^{-4}$: 5\%, 10\%, 100\% subsets]{\includegraphics[width=4.8cm]{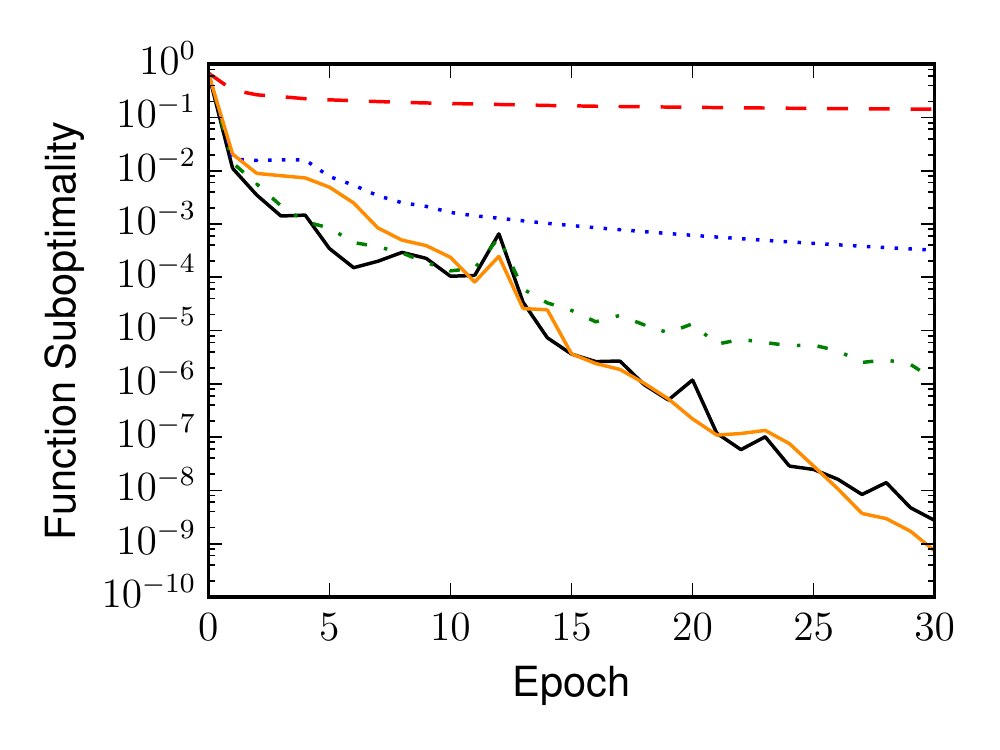}\includegraphics[width=4.8cm]{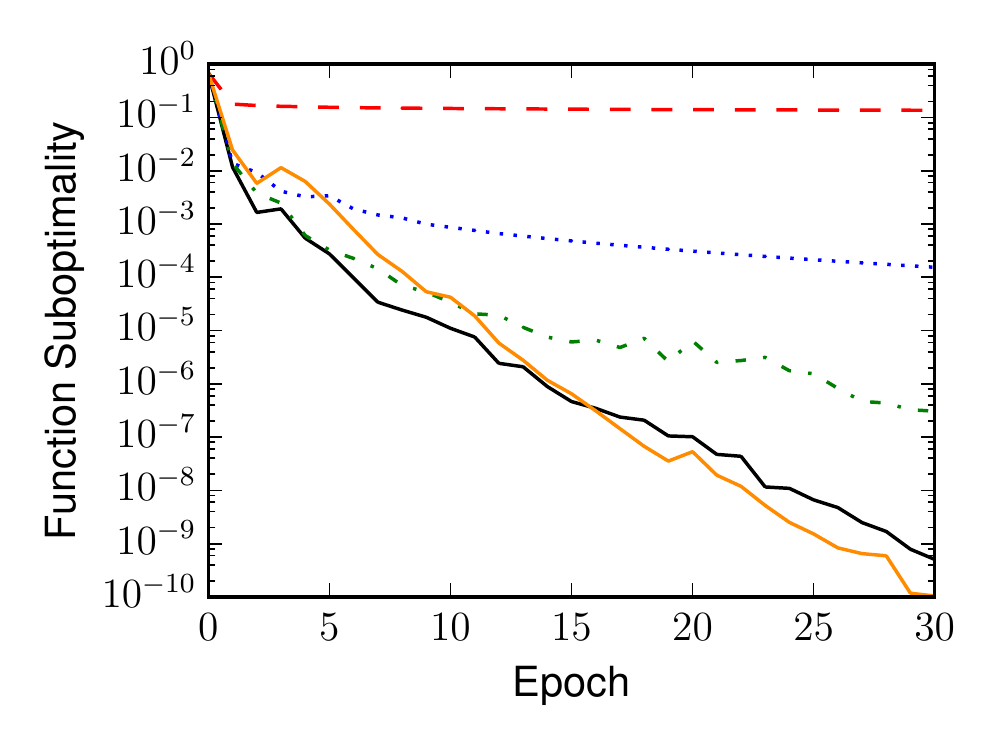}\includegraphics[width=4.8cm]{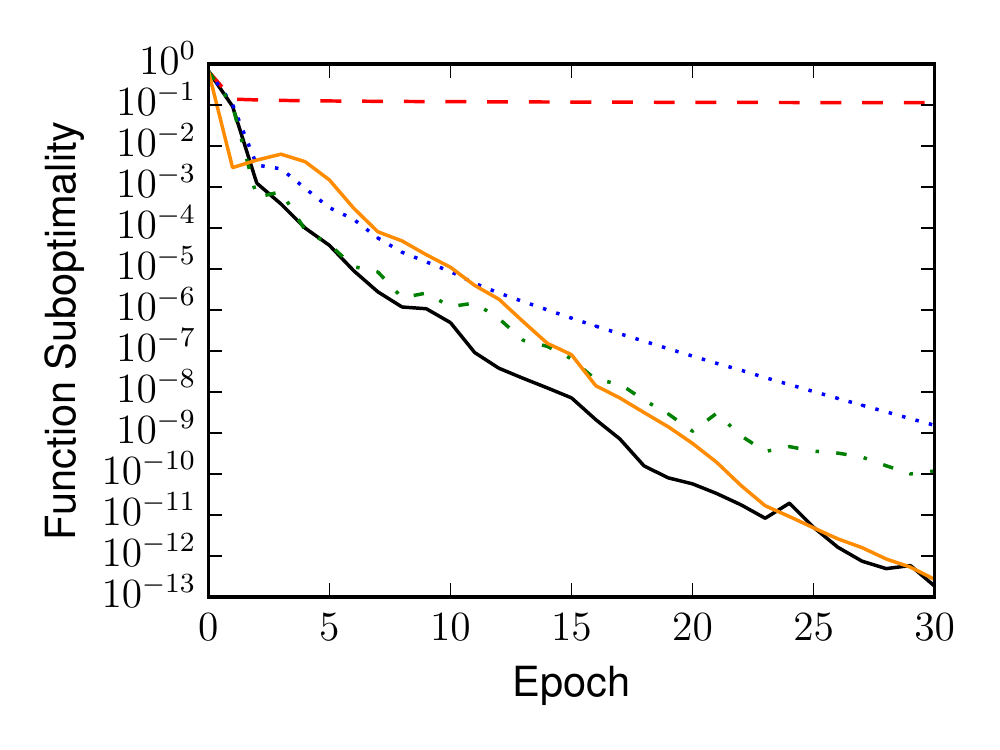}

}

\subfloat[RCV1 with hinge loss, $\mu=5\times10^{-5}$: 5\%, 10\%, 100\% subsets]{\includegraphics[width=4.8cm]{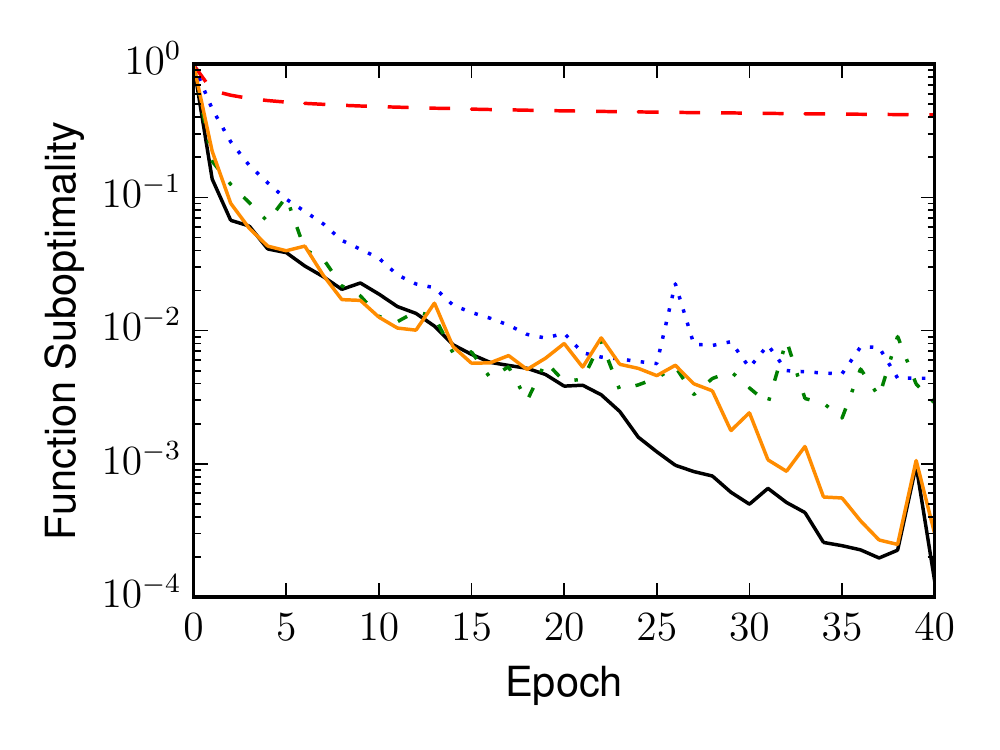}\includegraphics[width=4.8cm]{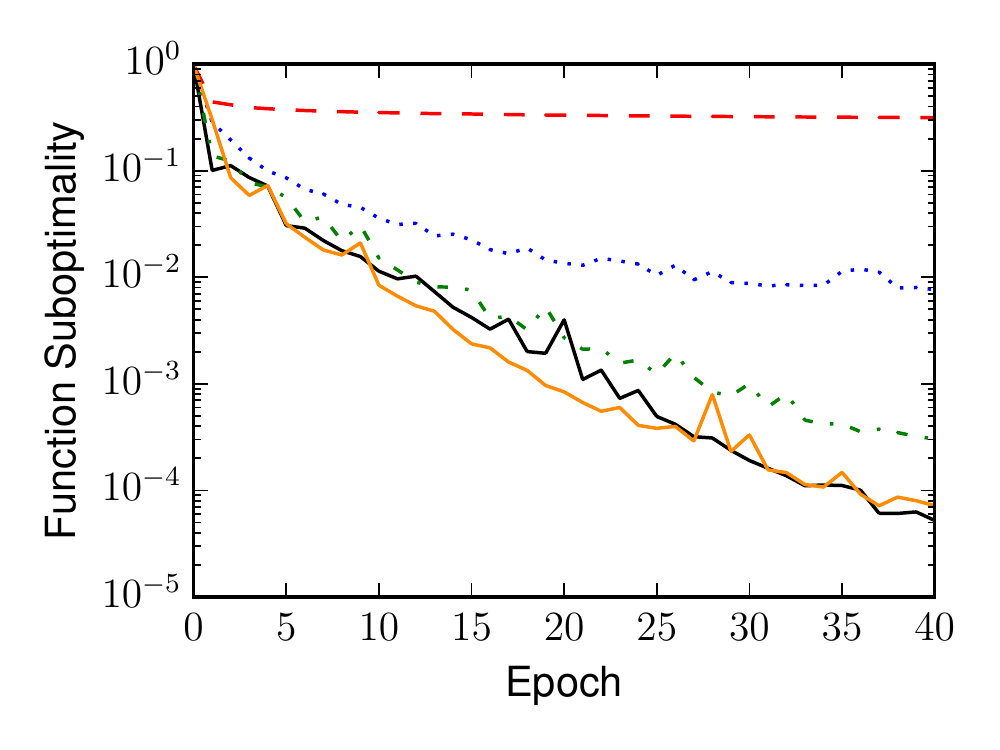}\includegraphics[width=4.8cm]{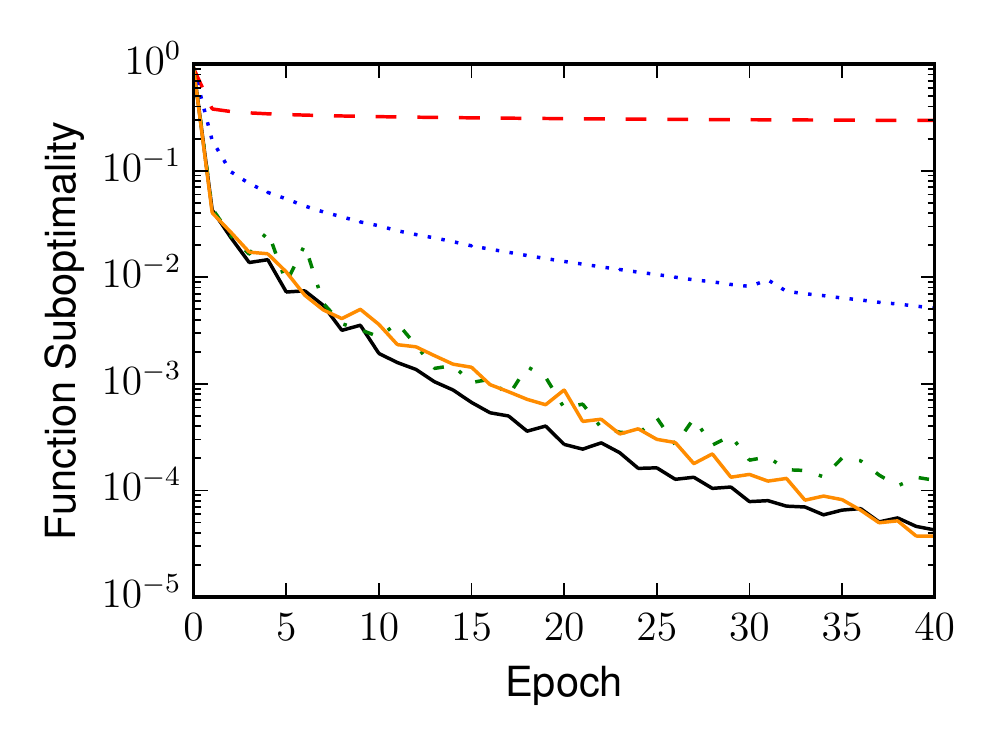}

}

\noindent \begin{centering}
\includegraphics{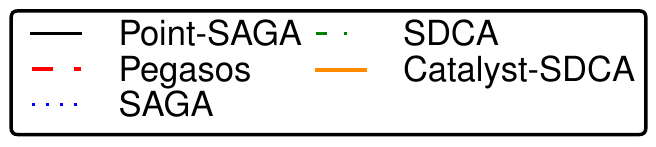}
\par\end{centering}

\caption{\label{fig:results}Experimental results}
\end{figure*}

\paragraph*{Non-smooth problems}

We also tested the RCV1 dataset on the hinge loss. In general we did
not expect an accelerated rate for this problem, and indeed we observe
that Point-SAGA is roughly as fast as SDCA across the different dataset
sizes. 

\newpage{}

\bibliographystyle{plainnat}
\bibliography{nips_2016}

\begin{thebibliography}{17}
\providecommand{\natexlab}[1]{#1}
\providecommand{\url}[1]{\texttt{#1}}
\expandafter\ifx\csname urlstyle\endcsname\relax
  \providecommand{\doi}[1]{doi: #1}\else
  \providecommand{\doi}{doi: \begingroup \urlstyle{rm}\Url}\fi

\bibitem[Chang and Lin(2011)]{libsvm}
Chih-Chung Chang and Chih-Jen Lin.
\newblock Libsvm : a library for support vector machines.
\newblock \emph{ACM Transactions on Intelligent Systems and Technology},
  2:27:1--27:27, 2011.

\bibitem[Defazio et~al.(2014{\natexlab{a}})Defazio, Bach, and
  Lacoste-Julien]{adefazio-nips2014}
Aaron Defazio, Francis Bach, and Simon Lacoste-Julien.
\newblock Saga: A fast incremental gradient method with support for
  non-strongly convex composite objectives.
\newblock \emph{Advances in Neural Information Processing Systems 27 (NIPS
  2014)}, 2014{\natexlab{a}}.

\bibitem[Defazio et~al.(2014{\natexlab{b}})Defazio, Caetano, and Domke]{finito}
Aaron Defazio, Tiberio Caetano, and Justin Domke.
\newblock Finito: A faster, permutable incremental gradient method for big data
  problems.
\newblock \emph{Proceedings of the 31st International Conference on Machine
  Learning}, 2014{\natexlab{b}}.

\bibitem[Hofmann et~al.(2015)Hofmann, Lucchi, Lacoste-Julien, and
  McWilliams]{lacoste-neighbors}
Thomas Hofmann, Aurelien Lucchi, Simon Lacoste-Julien, and Brian McWilliams.
\newblock Variance reduced stochastic gradient descent with neighbors.
\newblock In C.~Cortes, N.D. Lawrence, D.D. Lee, M.~Sugiyama, and R.~Garnett,
  editors, \emph{Advances in Neural Information Processing Systems 28}, pages
  2296--2304. Curran Associates, Inc., 2015.

\bibitem[Johnson and Zhang(2013)]{svrg}
Rie Johnson and Tong Zhang.
\newblock Accelerating stochastic gradient descent using predictive variance
  reduction.
\newblock \emph{NIPS}, 2013.

\bibitem[{Kone{\v c}n{\'y}} and {Richt{\'a}rik}(2013)]{semi}
Jakub {Kone{\v c}n{\'y}} and Peter {Richt{\'a}rik}.
\newblock {Semi-Stochastic Gradient Descent Methods}.
\newblock \emph{ArXiv e-prints}, December 2013.

\bibitem[{Lan} and {Zhou}(2015)]{lan-accel}
G.~{Lan} and Y.~{Zhou}.
\newblock {An optimal randomized incremental gradient method}.
\newblock \emph{ArXiv e-prints}, July 2015.

\bibitem[Lin et~al.(2015)Lin, Mairal, and Harchaoui]{mairal-catalyst}
Hongzhou Lin, Julien Mairal, and Zaid Harchaoui.
\newblock A universal catalyst for first-order optimization.
\newblock In C.~Cortes, N.D. Lawrence, D.D. Lee, M.~Sugiyama, and R.~Garnett,
  editors, \emph{Advances in Neural Information Processing Systems 28}, pages
  3366--3374. Curran Associates, Inc., 2015.

\bibitem[Mairal(2014)]{miso2}
Julien Mairal.
\newblock Incremental majorization-minimization optimization with application
  to large-scale machine learning.
\newblock Technical report, INRIA Grenoble Rh\^{o}ne-Alpes / LJK Laboratoire
  Jean Kuntzmann, 2014.

\bibitem[Nesterov(1998)]{nes-book}
Yu. Nesterov.
\newblock \emph{Introductory Lectures On Convex Programming}.
\newblock Springer, 1998.

\bibitem[Nitanda(2014)]{svrg-accel-nitanda}
Atsushi Nitanda.
\newblock Stochastic proximal gradient descent with acceleration techniques.
\newblock In Z.~Ghahramani, M.~Welling, C.~Cortes, N.D. Lawrence, and K.Q.
  Weinberger, editors, \emph{Advances in Neural Information Processing Systems
  27}, pages 1574--1582. Curran Associates, Inc., 2014.

\bibitem[Rockafellar(1976)]{rockafellar1976monotone}
R~Tyrrell Rockafellar.
\newblock Monotone operators and the proximal point algorithm.
\newblock \emph{SIAM journal on control and optimization}, 14\penalty0
  (5):\penalty0 877--898, 1976.

\bibitem[Schmidt et~al.(2013)Schmidt, Roux, and Bach]{SAG}
Mark Schmidt, Nicolas~Le Roux, and Francis Bach.
\newblock Minimizing finite sums with the stochastic average gradient.
\newblock Technical report, INRIA, 2013.

\bibitem[Shalev-Shwartz and Zhang(2013{\natexlab{a}})]{SDCA}
Shai Shalev-Shwartz and Tong Zhang.
\newblock Stochastic dual coordinate ascent methods for regularized loss
  minimization.
\newblock \emph{JMLR}, 2013{\natexlab{a}}.

\bibitem[Shalev-Shwartz and Zhang(2013{\natexlab{b}})]{accel-sdca}
Shai Shalev-Shwartz and Tong Zhang.
\newblock Accelerated mini-batch stochastic dual coordinate ascent.
\newblock In C.J.C. Burges, L.~Bottou, M.~Welling, Z.~Ghahramani, and K.Q.
  Weinberger, editors, \emph{Advances in Neural Information Processing Systems
  26}, pages 378--385. Curran Associates, Inc., 2013{\natexlab{b}}.

\bibitem[Shalev-Shwartz and Zhang(2013{\natexlab{c}})]{sdca-accel}
Shai Shalev-Shwartz and Tong Zhang.
\newblock Accelerated proximal stochastic dual coordinate ascent for
  regularized loss minimization.
\newblock Technical report, The Hebrew University, Jerusalem and Rutgers
  University, NJ, USA, 2013{\natexlab{c}}.

\bibitem[Shalev-Shwartz et~al.(2011)Shalev-Shwartz, Singer, Srebro, and
  Cotter]{pegasos}
Shai Shalev-Shwartz, Yoram Singer, Nathan Srebro, and Andrew Cotter.
\newblock Pegasos: Primal estimated sub-gradient solver for svm.
\newblock \emph{Mathematical programming}, 127\penalty0 (1):\penalty0 3--30,
  2011.

\end{thebibliography}

\newpage

\appendix

\section{Proximal operators}

For the most common binary classification and regression methods,
implementing the proximal operator is straight-forward. In this section
let $y_{j}$ be the label or target for regression, and $X_{j}$ the
data instance vector. We assume for binary classification that $y_{j}\in\{-1,1\}$.

\paragraph*{Hinge loss:}

\[
f_{j}(z)=l(z;y_{j},X_{j})=\max\left\{ 0,\,1-y_{j}\left\langle z,X_{j}\right\rangle \right\} .
\]
The proximal operator has a closed form expression:
\[
\text{prox}_{\gamma f_{j}}(z)=z-\gamma y_{j}\nu X_{j},
\]
where:
\[
s=\frac{1-y_{j}\left\langle z,X_{j}\right\rangle }{\gamma\left\Vert X_{j}\right\Vert ^{2}}.
\]
\[
\nu=\begin{cases}
-1 & s\geq1\\
0 & s\leq0\\
-s & \text{otherwise}
\end{cases}.
\]

\paragraph*{Logistic loss:}

\[
f_{j}(z)=l(z;y_{j},X_{j})=\log\left(1+\exp\left(-y_{j}X_{j}^{T}z\right)\right).
\]
There is no closed form expression, however it can be computed very
efficiently using Newton iteration, since it can be reduced to a 1D
minimization problem. In particular, let $c_{0}=0$, $\gamma^{\prime}=\gamma\left\Vert X_{j}\right\Vert ^{2}$,
and $a=\left\langle z,X_{j}\right\rangle $. Then iterate until convergence:
\[
s^{k}=\frac{-y_{j}}{1+\exp\left(y_{j}c^{k}\right)},
\]
\[
c^{k+1}=c^{k}-\frac{\gamma^{\prime}s^{k}+c^{k}-a}{1-y^{\prime}s^{k}-\gamma^{\prime}s^{k}s^{k}}.
\]
The prox operator is then $\text{prox}_{\gamma f_{j}}(z)=z-\left(a-c^{k}\right)X_{j}/\left\Vert X_{j}\right\Vert ^{2}$.
Three iterations are generally enough, but ill-conditioned problems
or large step sizes may require up to 12. Correct initialization is
important, as it will diverge when initialized with a point on the
opposite side of 0 from the solution.

\paragraph*{Squared loss:}

\[
f_{j}(z)=l(z;y_{j},X_{j})=\frac{1}{2}\left(X_{j}^{T}z-y_{j}\right)^{2}.
\]
Let $\gamma^{\prime}=\gamma\left\Vert X_{j}\right\Vert ^{2}$ and
$a=\left\langle z,X_{j}\right\rangle $. Define:
\[
c=\frac{a+\gamma^{\prime}y}{1+\gamma^{\prime}}.
\]
Then $\text{prox}_{\gamma f_{j}}(z)=z-\left(a-c\right)X_{j}/\left\Vert X_{j}\right\Vert ^{2}.$

\subsection*{L2 regularization}

Including a regularizer within each $f_{i}$, i.e. $F_{i}(x)=f_{i}(x)+\frac{\mu}{2}\left\Vert x\right\Vert ^{2},$
can be done using the proximal operator of $f_{i}$. Define the scaling
factor:
\[
\rho=1-\frac{\mu\gamma}{1+\mu\gamma}.
\]
Then $\text{prox}_{\gamma F_{i}}(z)=\text{prox}_{\rho\gamma f_{i}}(\rho z)$.

\section{Proofs}
\begin{lem}
\label{lem:inner-lemma}Under Algorithm 1, taking the expectation
over the random choice of $j$, conditioning on $x^{k}$ and each
$g_{i}^{k}$, allows us to bound the following inner product at step
$k$:
\begin{alignat*}{1}
 & E\left\langle \gamma\left[g_{j}^{k}-\frac{1}{n}\sum_{i=1}^{n}g_{i}^{k}\right]-\gamma g_{j}^{*},\left(x^{k}-x^{*}\right)+\gamma\left[g_{j}^{k}-\frac{1}{n}\sum_{i=1}^{n}g_{i}^{k}\right]-\gamma g_{j}^{*}\right\rangle \\
 & \leq\gamma^{2}\frac{1}{n}\sum_{i=1}^{n}\left\Vert g_{i}^{k}-g_{i}^{*}\right\Vert ^{2}.
\end{alignat*}
\end{lem}
\begin{proof}
We start by splitting on the right hand side of the inner product:
\begin{gather}
=E\left\langle \gamma\left[g_{j}^{k}-\frac{1}{n}\sum_{i=1}^{n}g_{i}^{k}\right]-\gamma g_{j}^{*}\,,\,x^{k}-x^{*}\right\rangle \nonumber \\
+E\left\langle \gamma\left[g_{j}^{k}-\frac{1}{n}\sum_{i=1}^{n}g_{i}^{k}\right]-\gamma g_{j}^{*}\,,\,\gamma\left[g_{j}^{k}-\frac{1}{n}\sum_{i=1}^{n}g_{i}^{k}\right]-\gamma g_{j}^{*}\right\rangle \label{eq:m1}
\end{gather}

The first inner product has expectation $0$ on the left hand side
(Recall that $E[g_{j}^{*}]=0$), so it's simply 0 in expectation (we
may take expectation on the left since the right doesn't depend on
$j$). The second inner product is the same on both sides, so we may
convert it to a norm-squared term. So we have:
\begin{eqnarray*}
 & = & \gamma^{2}E\left\Vert g_{j}^{k}-\frac{1}{n}\sum_{i=1}^{n}g_{i}^{k}-g_{j}^{*}\right\Vert ^{2}\\
 & \leq & \gamma^{2}E\left\Vert g_{j}^{k}-g_{j}^{*}\right\Vert ^{2}=\gamma^{2}\frac{1}{n}\sum_{i=1}^{n}\left\Vert g_{i}^{k}-g_{i}^{*}\right\Vert ^{2}.
\end{eqnarray*}

The inequality used is just an application of the variance formula
$E[\left(X-E[X]\right)^{2}]=E[X^{2}]-E[X]^{2}\leq E[X^{2}].$\end{proof}
\begin{cor}
Chaining the main theorem gives a convergence rate for point-saga
at step $k$ under the constants given in of:
\[
E\left\Vert x^{k}-x^{*}\right\Vert ^{2}\leq\left(1-\kappa\right)^{k}\frac{\mu+L}{\mu}\left\Vert x^{0}-x^{*}\right\Vert ^{2},
\]
if each $f_{i}:\mathbb{R}^{d}\rightarrow\mathbb{R}$ is $L$-smooth
and $\mu$-strongly convex.

\end{cor}
\begin{proof}
First we simplify $T^{0}$ using $c=1/\mu L$ and use Lipschitz smoothness:
\begin{eqnarray*}
T^{0} & = & \frac{1}{\mu L}\cdot\frac{1}{n}\sum_{i}\left\Vert g_{i}^{0}-g_{i}^{*}\right\Vert ^{2}+\left\Vert x^{0}-x^{*}\right\Vert ^{2}\\
 & \leq & \frac{L}{\mu}\cdot\left\Vert x^{0}-x^{*}\right\Vert ^{2}+\left\Vert x^{0}-x^{*}\right\Vert ^{2}\\
 & = & \frac{\mu+L}{\mu}\left\Vert x^{0}-x^{*}\right\Vert ^{2}.
\end{eqnarray*}
Now recall that the main theorem gives a bound $E\left[T^{k+1}\right]\leq\left(1-\kappa\right)T^{k}$
where the expectation is conditional on $x^{k}$ and each $g_{i}^{k}$
from step $k$, taking expectation over the randomness in the choice
of $j$. We can further take expectation with respect to $x^{k}$
and each $g_{i}^{k}$, giving the unconditional bound:
\[
E\left[T^{k+1}\right]\leq\left(1-\kappa\right)E\left[T^{k}\right].
\]

Chaining over $k$ gives the result.\end{proof}
\begin{thm}
Suppose each $f_{i}:\mathbb{R}^{d}\rightarrow\mathbb{R}$ is $\mu$-strongly
convex, $\left\Vert g_{i}^{0}-g_{i}^{*}\right\Vert \leq B$ and $\left\Vert x^{0}-x^{*}\right\Vert \leq R$.
Then after $k$ iterations of Point-SAGA with step size $\gamma=R/B\sqrt{n}$:
\[
E\left\Vert \bar{x}^{k}-x^{*}\right\Vert ^{2}\leq2\frac{\sqrt{n}\left(1+\mu\left(R/B\sqrt{n}\right)\right)}{\mu k}RB,
\]
where $\bar{x}^{k}=\frac{1}{k}E\sum_{t=1}^{k}x^{t}.$\end{thm}
\begin{proof}
Recall the bound on the Lyapunov function established in the main
theorem:
\begin{alignat*}{1}
E\left[T^{k+1}\right] & \leq T^{k}+\left(\alpha\gamma^{2}-\frac{c}{n}\right)\frac{1}{n}\sum_{i}^{n}\left\Vert g_{i}^{k}-g_{i}^{*}\right\Vert ^{2}\\
 & +\left(\frac{c}{n}-\alpha\gamma^{2}-\frac{\alpha\gamma}{L}\right)E\left\Vert g_{j}^{k+1}-g_{j}^{*}\right\Vert ^{2}\\
 & -\kappa E\left\Vert x^{k}-x^{*}\right\Vert ^{2}.
\end{alignat*}
In the non-smooth case this holds with $L=\infty$. In particular,
if we take $c=\alpha\gamma^{2}n$, then:
\[
-\kappa E\left\Vert x^{k+1}-x^{*}\right\Vert ^{2}\geq E\left[T^{k+1}\right]-T^{k}.
\]
Recall that this expectation is (implicitly) conditional on $x^{k}$
and each $g_{i}^{k}$ from step $k$, Taking expectation over the
randomness in the choice of $j$. We can further take expectation
with respect to $x^{k}$ and each $g_{i}^{k}$, and negate the inequality,
giving the unconditional bound:
\[
\kappa E\left\Vert x^{k+1}-x^{*}\right\Vert ^{2}\leq E\left[T^{k}\right]-E\left[T^{k+1}\right].
\]
We now sum this over $t=0\dots k$:
\[
\kappa E\sum_{t=1}^{k}\left\Vert x^{t}-x^{*}\right\Vert ^{2}\leq T^{0}-E\left[T^{k}\right].
\]
We can drop the $-E\left[T^{k}\right]$ since it is always negative.
Dividing through by $k$:
\[
\frac{1}{k}E\sum_{t=1}^{k}\left\Vert x^{t}-x^{*}\right\Vert ^{2}\leq\frac{1}{\kappa k}T^{0}.
\]
Now using Jensen's inequality on the left gives:
\[
E\left\Vert \bar{x}^{k}-x^{*}\right\Vert ^{2}\leq\frac{1}{\kappa k}T^{0},
\]
where $\bar{x}^{k}=\frac{1}{k}E\sum_{t=1}^{k}x^{t}.$ Now we plug
in $T^{0}=\frac{c}{n}\sum_{i}\left\Vert g_{i}^{0}-g_{i}^{*}\right\Vert ^{2}+\left\Vert x^{0}-x^{*}\right\Vert ^{2}$
with $c=\alpha\gamma^{2}n\leq\gamma^{2}n$:
\[
E\left\Vert \bar{x}^{k}-x^{*}\right\Vert ^{2}\leq\frac{\gamma^{2}n}{\kappa k}\frac{1}{n}\sum_{i}\left\Vert g_{i}^{0}-g_{i}^{*}\right\Vert ^{2}+\frac{1}{\kappa k}\left\Vert x^{0}-x^{*}\right\Vert ^{2}.
\]
Now we plug in the bounds in terms of $B$ and $R$:
\[
E\left\Vert \bar{x}^{k}-x^{*}\right\Vert ^{2}\leq\frac{\gamma^{2}n}{\kappa k}B^{2}+\frac{1}{\kappa k}R^{2}.
\]

In order to balance the terms on the right, we need:
\[
\frac{\gamma^{2}n}{\kappa k}B^{2}=\frac{1}{\kappa k}R^{2},
\]
\[
\therefore\gamma^{2}nB^{2}=R^{2},
\]
\[
\therefore\gamma^{2}=\frac{R^{2}}{nB^{2}}.
\]

So we can take $\gamma=R/B\sqrt{n}$, giving a rate of:
\begin{eqnarray*}
E\left\Vert \bar{x}^{k}-x^{*}\right\Vert ^{2} & \leq & \frac{2}{\kappa k}R^{2}\\
 & = & 2\frac{1+\mu\gamma}{\mu\gamma k}R^{2}\\
 & = & 2\frac{\sqrt{n}\left(1+\mu\left(R/B\sqrt{n}\right)\right)}{\mu k}RB.
\end{eqnarray*}

\end{proof}

\section{Proximal operator bounds with proofs}

In this section we prove some simple bounds from proximal operator
theory that we will use in this work. Define the short-hand $p_{\gamma f}(x)=\text{prox}_{\gamma f}(x)$,
and let $g_{\gamma f}(x)=\frac{1}{\gamma}\left(x-p_{\gamma f}(x)\right)$,
so that $p_{\gamma f}(x)=x-\gamma g_{\gamma f}(x)$. Note that $g_{\gamma f}(x)$
is a subgradient of $f$ at the point $p_{\gamma f}(x)$. This relation
is known as the optimality condition of the proximal operator.

We will also use a few standard convexity bounds without proof. Let
$f:\mathbb{R}^{d}\rightarrow\mathbb{R}$ be a convex function with
strong convexity constant $\mu\geq0$ and Lipschitz smoothness constant
$L$. Let $x^{*}$ be the minimizer of $f$, then for any $x,y\in\mathbb{R}^{d}$:
\begin{equation}
\left\langle f^{\prime}(x)-f^{\prime}(y),x-y\right\rangle \geq\mu\left\Vert x-y\right\Vert ^{2},\label{eq:sc-ip-bound}
\end{equation}
\begin{equation}
\left\Vert f^{\prime}(x)-f^{\prime}(y)\right\Vert ^{2}\leq L^{2}\left\Vert x-y\right\Vert ^{2}.\label{eq:lip-bound}
\end{equation}

\begin{prop}
\label{thm:firm-nonexpansiveness}(Firm non-expansiveness) For any
$x,y\in\mathbb{R}^{d}$, and any convex function $f:\mathbb{R}^{d}\rightarrow\mathbb{R}$
with strong convexity constant $\mu\geq0$, 
\[
\left\langle x-y,p_{\gamma f}(x)-p_{\gamma f}(y)\right\rangle \geq(1+\mu\gamma)\left\Vert p_{\gamma f}(x)-p_{\gamma f}(y)\right\Vert ^{2}.
\]
\end{prop}
\begin{proof}
Using strong convexity of $f,$ we apply Equation \ref{eq:sc-ip-bound}
at the (sub-)gradients $g_{\gamma f}(x)$ and $g_{\gamma f}(y)$,
and their corresponding points $p_{\gamma f}(x)$ and $p_{\gamma f}(y)$:
\[
\left\langle g_{\gamma f}(x)-g_{\gamma f}(y),p_{\gamma f}(x)-p_{\gamma f}(y)\right\rangle \geq\mu\left\Vert p_{\gamma f}(x)-p_{\gamma f}(y)\right\Vert ^{2}.
\]

We now multiply both sides by $\gamma$, then add $\left\Vert p_{\gamma f}(x)-p_{\gamma f}(y)\right\Vert ^{2}$
to both sides:
\[
\left\langle p_{\gamma f}(x)+\gamma g_{\gamma f}(x)-p_{\gamma f}(y)-\gamma g_{\gamma f}(y),p_{\gamma f}(x)-p_{\gamma f}(y)\right\rangle \geq\left(1+\mu\gamma\right)\left\Vert p_{\gamma f}(x)-\text{p}_{\gamma f}(y)\right\Vert ^{2},
\]

leading to the bound by using the optimality condition: $p_{\gamma f}(x)+\gamma g_{\gamma f}(x)=x.$\end{proof}
\begin{prop}
(Moreau decomposition) For any $x\in\mathbb{R}^{d}$, and any convex
function $f:\mathbb{R}^{d}\rightarrow\mathbb{R}$ with Fenchel conjugate
$f^{*}$ :
\begin{equation}
p_{\gamma f}(x)=x-\gamma p_{\frac{1}{\gamma}f^{*}}(x/\gamma).\label{eq:Moreau}
\end{equation}

Recall our definition of $g_{\gamma f}(x)=\frac{1}{\gamma}\left(x-p_{\gamma f}(x)\right)$
also. After combining, the following relation thus holds between the
proximal operator of the conjugate $f^{*}$ and $g_{\gamma f}$:
\begin{equation}
p_{\frac{1}{\gamma}f^{*}}(x/\gamma)=\frac{1}{\gamma}\left(x-p_{\gamma f}(x)\right)=g_{\gamma f}(x).\label{eq:gconj}
\end{equation}
\end{prop}
\begin{proof}
Let $u=p_{\gamma f}(x)$, and $v=\frac{1}{\gamma}\left(x-u\right)$.
Then $v\in\partial f(u)$ by the optimality condition of the proximal
operator of $f$ (namely if $u=p_{\gamma f}(x)$ then $u=x-\gamma v\Leftrightarrow v\in\partial f(u)$).
It follows by conjugacy of $f$ that $u\in\partial f^{*}(v).$ Thus
we may interpret $v=\frac{1}{\gamma}\left(x-u\right)$ as the optimality
condition of a proximal operator of $f^{*}$ :
\[
v=p_{\frac{1}{\gamma}f^{*}}(\frac{1}{\gamma}x).
\]

Plugging in the definition of $v$ then gives:
\[
\frac{1}{\gamma}\left(x-u\right)=p_{\frac{1}{\gamma}f^{*}}(\frac{1}{\gamma}x).
\]

Further plugging in $u=p_{\gamma f}(x)$ and rearranging gives the
result. \end{proof}

\end{document}